%% file: main.tex
\documentclass[11pt]{article}
\usepackage{fullpage}
\usepackage[algo2e,ruled]{algorithm2e}

\usepackage{natbib}
\usepackage[hyphens]{url}

\usepackage{microtype}
\usepackage{graphicx}
\usepackage{todonotes}
\usepackage{booktabs} 

\usepackage{hyperref}
\usepackage{nicefrac}
\usepackage{listings}
\usepackage{algorithmic}
\usepackage{paralist}

\usepackage[utf8]{inputenc} 
\usepackage[T1]{fontenc}    
\usepackage{hyperref}       
\usepackage{url}            
\usepackage{booktabs}       
\usepackage{amsfonts}       
\usepackage{nicefrac}       
\usepackage{microtype}      
\usepackage{xcolor}         

 \usepackage{algorithm}
\usepackage{amsmath}
\usepackage{amssymb}
\usepackage{mathtools}
\usepackage{amsthm}

\usepackage{subcaption}
\usepackage{graphicx}

\usepackage[capitalize,noabbrev]{cleveref}

\theoremstyle{plain}
\newtheorem{theorem}{Theorem}[section]

\newtheorem{lemma}[theorem]{Lemma}
\newtheorem{corollary}[theorem]{Corollary}
\theoremstyle{definition}
\newtheorem{definition}[theorem]{Definition}
\newtheorem{assumption}{Assumption}
\theoremstyle{remark}

\input{preamble}

\begin{document}

\title{Noise is All You Need: Private Second-Order Convergence of Noisy SGD}
\author{%
Dmitrii Avdiukhin \thanks{dmitrii.avdiukhin@northwestern.edu,
Northwestern University. }
\and
Michael Dinitz \thanks{ mdinitz@cs.jhu.edu, Johns Hopkins University.}
\and
Chenglin Fan \thanks{ fanchenglin@gmail.com,  Seoul National University. }
\and
Grigory Yaroslavtsev 
\thanks{grigory@gmu.edu,
George Mason University.}
}

\newpage

\setcounter{page}{1}
\maketitle




\input{abstract}
\input{intro}
\input{prelims}

\input{algorithm}

\input{experiments}
\input{conclusion}

\bibliography{main}
\bibliographystyle{plainnat}

\newpage
\appendix
\onecolumn
\input{experimental_details}
\input{proof_convergence}

\input{proofs_no_nSG}

\end{document}

%% file: preamble.tex
\newcommand{\R}{\mathbb R}
\newcommand{\E}{\mathbb E}
\newcommand{\eps}{\varepsilon}

\newcommand{\abs}[1]{\left| #1 \right|}
\newcommand{\norm}[1]{\left\| #1 \right\|}
\newcommand{\SqrNorm}[1]{\left\| #1 \right\|^2}
\newcommand{\pars}[1]{\left( #1 \right)}
\newcommand{\sqbr}[1]{\left[ #1 \right]}

\newcommand{\InnerProd}[1]{\left\langle #1 \right\rangle}
\newcommand{\Exp}[1]{\E \sqbr{#1}}
\newcommand{\Var}[1]{\mathbf{Var} \sqbr{#1}}
\newcommand{\ExpSqrNorm}[1]{\E \norm{#1}^2}

\newcommand{\tOh}[1]{\tilde{O} \pars{#1}}

\newcommand{\tOm}[1]{\tilde{\Omega} \pars{#1}}

\newcommand{\vzero}{\mathbf{0}}

\newcommand{\vg}{\mathbf g}

\newcommand{\vv}{\mathbf v}

\newcommand{\vx}{\mathbf x}
\newcommand{\vy}{\mathbf y}

\newcommand{\vH}{\mathbf H}

\DeclareMathOperator*{\avg}{avg}
\DeclareMathOperator{\minimize}{minimize}

\newcommand{\x}[1]{\vx_{#1}}
\newcommand{\tx}[1]{\vx'_{#1}}
\newcommand{\difx}[1]{\bar{\vx}_{#1}}

\newcommand{\ourNoise}[1]{\xi_{#1}}
\newcommand{\tourNoise}[1]{\xi'_{#1}}
\newcommand{\difourNoise}[1]{\bar{\xi}_{#1}}

\newcommand{\sgdNoise}[1]{\zeta_{#1}}
\newcommand{\tsgdNoise}[1]{\zeta'_{#1}}
\newcommand{\difsgdNoise}[1]{\bar{\zeta}_{#1}}

\newcommand{\sgdErr}[1]{\mathrm{err}^{\mathrm{sgd}}_{#1}}
\newcommand{\quadErr}[1]{\mathrm{err}^{\mathrm{Q}}_{#1}}
\newcommand{\dif}[1]{\mathrm{dif}_{#1}}
\newcommand{\hesInt}[1]{\Gamma_{#1}}
\newcommand{\logErr}{\log \nicefrac{1}{\tdelta}}

\newcommand{\fmax}{f_{\max}}
\newcommand{\batch}[1]{\mathcal{B}_{#1}}
\newcommand{\g}[1]{\vg_{#1}}
\newcommand{\id}{\mathbf{I}}
\newcommand{\tdelta}{\tilde \delta}
\newcommand{\tsigma}{\tilde \sigma}
\newcommand{\fullVar}{\tilde \sigma}
\newcommand{\step}{\eta}

\renewcommand{\P}[1]{\mathbb P \sqbr{#1}}

\newcommand{\sqrtra}{\sqrt{\rho \alpha}}
\newcommand{\sqrtfacr}{\sqrt{\alpha^3 / \rho}}

\newcommand{\escIter}{\mathcal I}
\newcommand{\df}{\mathcal F}
\newcommand{\escRad}{\mathcal R}
\newcommand{\eigen}{\gamma}
\newcommand{\lmin}{\lambda_{\min}}

\newcommand{\cstep}{c_{\step}}
\newcommand{\cit}{c_{\escIter}}
\newcommand{\crad}{c_{\escRad}}
\newcommand{\cf}{c_{\df}}

\newcommand{\subfigwidth}{0.48}

%% file: abstract.tex
\begin{abstract}
Private optimization is a topic of major interest in machine learning, with differentially private stochastic gradient descent (DP-SGD) playing a key role in both theory and practice. Furthermore, DP-SGD is known to be a powerful tool in contexts beyond privacy, including robustness, machine unlearning, etc.
Existing analyses of DP-SGD either make relatively strong assumptions (e.g., Lipschitz continuity of the loss function, or even convexity) or prove only first-order convergence (and thus might end at a saddle point in the non-convex setting).
At the same time, there has been progress in proving second-order convergence of the non-private version of ``noisy SGD'', as well as progress in designing algorithms that are more complex than DP-SGD and do guarantee second-order convergence.
We revisit DP-SGD and show that ``noise is all you need'': the noise necessary for privacy already implies second-order convergence under the standard smoothness assumptions, even for non-Lipschitz loss functions.
Hence, we get second-order convergence essentially for free: DP-SGD, the workhorse of modern private optimization, under minimal assumptions can be used to find a second-order stationary point.
\end{abstract}

%% file: intro.tex
\section{Introduction}


In recent years, there has been a growing interest in private data analysis, from academic research to real-world applications.
\emph{Differential privacy}~\citep{dwork_CalibratingNoise_2006} provides strong privacy guarantee for each individual, while allowing to train accurate machine learning models.
Informally, an algorithm is differentially private if, for any person, the outcome of the algorithm is roughly the same regardless of whether or not this person's data is included.
Differential privacy has become one of the most popular formalizations of privacy in both theory and practice~\citep{dwork_AlgorithmicFoundations_2014,DBLP:series/synthesis/2016Li}, and has been adopted in a variety of settings including regression,  combinatorial optimization, continuous optimization, principal component analysis, matrix completion, reinforcement learning  and deep learning~\citep{chaudhuri_PrivacypreservingLogistic_2008,chaudhuri_DifferentiallyPrivate_2011,DBLP:conf/soda/GuptaLMRT10,agarwal_CpSGDCommunicationefficient_2018,ge_MinimaxOptimalPrivacyPreserving_2018,jain_DifferentiallyPrivate_2018, DBLP:conf/icml/VietriBKW20, abadi_DeepLearning_2016}.


A typical machine learning problem can be stated as \emph{empirical risk minimization} (ERM), where we optimize over the parameters of the model in order to minimize the empirical risk. 
Formally, ERM is:
\begin{align*}
    \minimize_{ w \in \R^d } f( w ) := \frac{1}{n} \sum_{i=1}^n f_i(w),
\end{align*}
where $n$ is the number of samples, $f_i$'s are differentiable functions and $w$ are the model weights.

Since sampled data points typically correspond to information about individuals, it is natural to ask whether we can do ERM \emph{privately}~\citep{bassily_PrivateEmpirical_2014}.  While there are various algorithmic approaches for solving ERM, Stochastic Gradient Descent (SGD) plays a key role as a very basic yet widely used practical algorithm. SGD is a first-order optimization algorithm that updates the model parameters based on the gradient of the loss function with respect to these parameters.
Due to the simplicity and efficiency of SGD, one of the most important techniques in private optimization is Differentially Private SGD (DP-SGD) (see e.g.~\citet{abadi_DeepLearning_2016}), which is implemented as a part of TensorFlow library and in PyTorch\footnote{\url{https://github.com/tensorflow/privacy}, \url{https://github.com/pytorch/opacus/}}.
DP-SGD guarantees differential privacy by adding noise to the gradients computed during the training process. By carefully controlling the noise, this can be achieved without substantially sacrificing the overall utility of the trained model.
DP-SGD now plays a central role in private optimization, and has been studied extensively in recent years~\citep{song2013stochastic,bassily_PrivateEmpirical_2014,abadi_DeepLearning_2016,wang2017differentially,bassily2019private,feldman2020private,asi2021private,yang2022normalized,DBLP:conf/iclr/FangLF023}.

SGD converges to a point with a small gradient, which can be a local minimum or a saddle point.
In general, finding a local minimum is NP-hard for non-convex functions~\citep{DBLP:conf/colt/AnandkumarG16}.
Despite this, the problem of finding a local minimum by escaping from saddle points has received significant attention~\citep{DBLP:conf/colt/GeHJY15,DBLP:conf/nips/BhojanapalliNS16, DBLP:conf/icml/0001JZ17, jin_HowEscape_2017,carmon_AnalysisKrylov_2018,jin_AcceleratedGradient_2018,allen-zhu_NatashaFaster_2018,allen-zhu_NEON2Finding_2018,fang_SPIDERNearOptimal_2018,tripuraneni_StochasticCubic_2018,xu_FirstorderStochastic_2018, daneshmand_EscapingSaddles_2018, carmon_FirstOrderMethods_2020,zhou_StochasticRecursive_2020,zhou_StochasticNested_2020, jin_NonconvexOptimization_2021, zhang_EscapeSaddle_2021}.
Many non-convex ERM functions have been shown to be strict saddles~\citep{DBLP:conf/colt/GeHJY15}, meaning that a second-order stationary point (SOSP) (i.e. a point with a small gradient and a bounded Hessian) can be sufficiently  close to a local minimum. Similarly to differential privacy, adding noise to the gradient during training has been the key technique for escaping from saddle points~\citep{DBLP:conf/colt/GeHJY15,jin_NonconvexOptimization_2021}. 

Since adding noise to gradients during training is the fundamental idea behind both private SGD and second-order convergence of SGD, it is natural to ask whether ``noise is all you need'': is adding noise enough to guarantee both privacy and second-order convergence simultaneously?  Or do we need more sophisticated algorithms to simultaneously guarantee privacy and second-order convergence?
There are a number of technical complications that arise when trying to analyze both properties. Most notably, standard versions of DP-SGD~\citep{abadi_DeepLearning_2016} use a technique known as \emph{gradient clipping}, which goes beyond simply adding noise.
We ask a simple yet fundamental question:
\begin{center}\textit{Does noisy SGD converge to a second-order stationary point while also providing privacy?}\end{center}





\subsection{Our Results}

In this paper, we give an affirmative answer the above question: we show that an extremely simple version of noisy SGD is enough to find a second-order stationary point privately.
We consider a version of DP-SGD which does not use gradient clipping~-- the usual technique used to ensure privacy~-- and doesn't assume bounded domain.
Instead, we run vanilla SGD, but with carefully chosen noise added to each stochastic gradient. 
Our main result is the following, where $\vx$ is an $\alpha$-second order stationary point ($\alpha$-SOSP) for a twice-differentiable function $f$ if $\|\nabla f(\vx)\| < \alpha$ and $\lmin(\nabla^2 f(\vx)) > -\sqrtra$ (see Definition~\ref{def:sosp}):

\begin{theorem}[Informal, see Theorem~\ref{thm:combine_main}]
    
    Under the standard assumptions of non-convex optimization~-- in particular, without the Lipschitz condition~-- noisy SGD (without gradient clipping) is $(\eps,\delta)$-differentially private and for any $\alpha = \tilde{\Omega} \pars{\nicefrac{d^{1/4}}{\sqrt{n \eps}}}$ finds an $\alpha$-second-order stationary point w.h.p.  Moreover, the number of stochastic oracle calls with variance $\sigma^2$ is
    $\tOh{\nicefrac{\sigma^2}{\alpha^4}}$.
    
\end{theorem}


In non-private optimization, an arbitrary small $\alpha$ can be achieved at the cost of larger running time, with $\tOh{\nicefrac{\sigma^2}{\alpha^4}}$ being a standard number of stochastic oracle calls~\citep{arjevani_LowerBounds_2023}.
On the other hand, in differentially private optimization, there is a trade-off between privacy and accuracy, with $\nicefrac{d^{1/4}}{\sqrt{n \eps}}$ being the standard convergence rate.
Our result has a number of important differences when compared to recent related work on privately finding second-order stationary points.

\paragraph{Gradient clipping}
Our algorithm is very simple: it is essentially just ``vanilla'' SGD but with noise added to preserve privacy.
Contrary to that, the standard version of DP-SGD (e.g., from~\citet{abadi_DeepLearning_2016}) uses a technique known as \emph{gradient clipping}, where if the algorithm encounters a large gradient, it first scales the gradient to have a small norm before adding noise.
In general, clipping introduces non-vanishing bias term hurting convergence (see \citet{koloskova_RevisitingGradient_2023,chen_UnderstandingGradient_2020} for examples and convergence rates).
In contrast, based on ideas from~\citet{DBLP:conf/iclr/FangLF023}, we design a simpler differentially private algorithm where clipping is not necessary.

\paragraph{More involved algorithms.}
Most notably, there is a recent line of work on this problem (private second-order convergence), with the goal of achieving optimal (or near-optimal) rates.
This requires using more involved algorithms, often based on SPIDER~\citep{fang_SPIDERNearOptimal_2018} or SpiderBoost~\citep{wang_SpiderBoostMomentum_2019}.
In particular,~\citet{DBLP:conf/icml/AroraBGGMU23} introduced a private version SpiderBoost to find second-order stationary points privately, later refined by~\citet{liu2023private}. While these are important results, they require more involved algorithms, which are rarely used in practice.
SGD and DP-SGD are still the workhorses of modern (private) machine learning. Understanding their powers and limitations is the key motivation of our work.

\paragraph{Stronger Assumptions, Weaker Algorithms.} 
\citet{wang_DifferentiallyPrivate_2019} analyze simple first-order algorithms for private second-order convergence. They give upper bounds for private ERM in a number of settings, including finding approximate local minima. Our bounds offer two important improvements.

First,~\citet{wang_DifferentiallyPrivate_2019} only consider \emph{non-stochastic} gradients.
In other words, rather than computing stochastic gradients through random minibatches, in every iteration they use the entire dataset to compute an exact gradient. 
In practice, computing full gradients is usually prohibitively expensive.
Second, we make less restrictive assumptions on the loss functions: \citet{wang_DifferentiallyPrivate_2019}, as common in differentially private optimization, assume that the loss functions are Lipschitz continuous, while we assume that they are smooth~-- a standard assumption in non-convex optimization~-- and have a Lipschitz Hessian~-- a standard assumption for the second-order convergence~\citep{jin_NonconvexOptimization_2021,fang_SPIDERNearOptimal_2018}.
Replacing a Lipschitz assumption~\citep{Searcoid2006} with less restrictive assumptions has recently been studied for \emph{first}-order convergence of DP-SGD~\citep{DBLP:conf/iclr/FangLF023}; we continue this line of work by generalizing it to second-order convergence.
As in~\citet{DBLP:conf/iclr/FangLF023}, since we are making very limited assumptions about the loss function, it is not possible to show global optimality, unlike in~\citet{wang_DifferentiallyPrivate_2019}, who are able to do so by assuming Lipschitzness and regularization.

\paragraph{Broader Impact}
We show that noisy SGD without clipping finds second-order stationary points, which are good approximations of local minima.
In many cases, local minima are global minima (see e.g.~\citet{sun_CompleteDictionary_2017,ge_MatrixCompletion_2016a,kawaguchi_DeepLearning_2016}), and hence DP-SGD is capable of finding high-quality solutions while preserving privacy.

One of the main privacy applications is federated learning~\citep{kairouz_AdvancesOpen_2021}, where multiple clients cooperate in training a single model.
Our result implies that noisy distributed SGD preserves the privacy of its clients.
Our result shows that noisy SGD is a good choice for major areas with sensitive data, including healthcare~\citep{ficek2021differential,hassan_DifferentialPrivacy_2020}, Internet of Things~\citep{jiang_DifferentialPrivacy_2021}, geolocation~\citep{andres_GeoindistinguishabilityDifferential_2013}, etc.
For certain applications, such as natural language processing, differential privacy is crucial due to models potentially memorizing the inputs~\citep{carlini_SecretSharer_2019a}.

Aside from the above direct applications, differential privacy guarantee are useful for machine unlearning~\citep{sekhari_RememberWhat_2021}, which is an important direction, in particular, due to General Data Protection Regulation.
Since differential privacy guarantees that output is not sensitive to any given example, the model trained by noisy SGD is guaranteed to be able to unlearn a noticeable fraction of inputs.
Similarly, low sensitivity to individual samples makes the model trained with noisy SGD robust~\citep{lecuyer_CertifiedRobustness_2019}, i.e. a significant number of adversarial examples is needed to noticeably affect the model.
Moreover, such models have provable bounds on generalization~\citep{jung_NewAnalysis_2024}.

\subsection{Technical Overview}
 
As mentioned above, our algorithm is simply a ``vanilla'' SGD with noise.
Despite this simplicity, there are some interesting and important subtleties.
At a very high level, the noise we have to add to preserve privacy depends on the magnitude of the gradients we see in the algorithm.
The standard solution is to use gradient clipping, which guarantees that gradient norm is small compared with noise.
We, on the other hand, use the fact that standard assumptions in non-convex optimization imply that the gradient norms can be uniformly bounded depending on the initial objective value (as pointed out in~\citet{DBLP:conf/iclr/FangLF023} for private first-order convergence) to set the noise in such a way that clipping is essentially not necessary.

The above observation simplifies our analysis significantly: we can use the existing privacy analysis of DP-SGD with clipping almost as a black box, and can focus on proving convergence to a second-order stationary point of ``noisy'' SGD with no complications from clipping.  Our analysis of convergence has a number of technical subtleties, but is overall inspired by the non-private analysis of the second-order convergence of noisy SGD.  The main difference is that the noise we add for privacy is \emph{significantly} larger than the noise that we would have to add if our only goal was to escape from saddle points. 
Hence, proving convergence requires carefully arguing that this added noise is still dominated by the fast convergence of SGD.  


\subsection{Related Work}

\textbf{Convergence Analysis of DP-SGD.}  
DP-SGD with gradient clipping was initially proposed by~\citet{abadi_DeepLearning_2016}.
Asymptotically, under the Lipschitzness assumption, DP-SGD was shown to produce a tight utility-privacy tradeoff~\citep{bassily2019private}.
However, the understanding of DP-SGD remains limited without Lipschitz continuity.
The analysis of convergence rate of DP-SGD without Lipschitzness and bounded domain assumptions is a challenging task~\citep{wang2022differentially}.
On one hand, negative examples are shown in ~\citet{chen_UnderstandingGradient_2020,song2021evading} where DP-SGD in general might not converge. 
On the other hand, \citet{DBLP:conf/iclr/FangLF023} showed that DP-SGD (with or without clipping) converges in the unbounded domain without  Lipschitzness.
The main idea of~\citet{DBLP:conf/iclr/FangLF023} is to show that w.h.p. clipping never actually happens, at least under the widely used light-tail assumption (Assumption~\ref{ass:grad_concentration}).
Another work by~\citet{yang2022normalized} studied the convergence of DP-SGD under the generalized smoothness condition.
The concurrent work by~\citet{bu2023automatic} shows that a small clipping threshold can
yield promising performance in training language models.

The original implementation~\citep{abadi_DeepLearning_2016} of gradient clipping can be computationally inefficient as one needs to calculate the norm of the gradient for each individual sample in every iteration.
Due to its importance and centrality in private optimization and machine learning, there has been significant work on improving the efficiency of DP-SGD from both algorithmic and engineering perspectives; see~\citet{goodfellow2015efficient,abadi_DeepLearning_2016,rochette2019efficient,bu2021fast,subramani2021enabling} for some examples of such work. 

\textbf{Analysis of (private) convergence to an $\alpha$-SOSP.}
An \emph{$\alpha$-first-order stationary point} ($\alpha$-FOSP, Definition~\ref{def:fosp}) is point $\vx$ so that $\|\nabla f(\vx)\| \le \alpha$.
Analyses of convergence to an $\alpha$-FOSP are a cornerstone of non-convex optimization (see, e.g., classical texts~\citet{bertsekas_NonlinearProgramming_1997,nocedal_NumericalOptimization_1999}).
Quantitative analysis of convergence to an $\alpha$-second-order stationary point ($\alpha$-SOSP, Definition~\ref{def:sosp}) started with the breakthrough work by~\citet{DBLP:conf/colt/GeHJY15}, further refined by~\citet{jin_HowEscape_2017,carmon_AnalysisKrylov_2018,jin_AcceleratedGradient_2018,carmon_FirstOrderMethods_2020,jin_NonconvexOptimization_2021} and most recently in~\citet{zhang_EscapeSaddle_2021}, who show an almost optimal bound. 
Due to the prevalence of SGD in deep learning, stochastic methods have attracted the most attention (see~\citet{allen-zhu_NatashaFaster_2018,allen-zhu_NEON2Finding_2018,fang_SPIDERNearOptimal_2018,tripuraneni_StochasticCubic_2018,xu_FirstorderStochastic_2018,zhou_StochasticRecursive_2020,zhou_StochasticNested_2020} for the case of Lipschitz gradients and~\citet{DBLP:conf/colt/GeHJY15,daneshmand_EscapingSaddles_2018} for non-Lipschitz gradients).
For an in-depth summary of the previous work on unconstrained non-convex optimization, we refer the reader to~\citet{jin_NonconvexOptimization_2021}.

For private second-order convergence, a private version of SpiderBoost was proposed by~\citet{DBLP:conf/icml/AroraBGGMU23} and was then refined by~\citet{liu2023private}.
Since here we are interested in understanding the fundamental power of SGD rather than in obtaining optimal convergence rates, the main baseline we compare with is ~\citet{wang_DifferentiallyPrivate_2019,DBLP:conf/pkdd/WangX20}, which essentially uses noisy gradient descent.
Compared to this result, we allow stochastic gradients and make weaker assumptions about the structure of the problem. 


\textbf{Roadmap.} In Section~\ref{sec:pre} we describe the basic definitions and assumptions.  Section~\ref{sec:algorithm} describes our version of  DP-SGD without clipping, as well as the privacy analysis and the outline of the convergence analysis. Empirical evaluations are provided in section~\ref{sec:experiments}, with additional experimental results provided in Appendix~\ref{app:experimental_details}.
The proof details excluded from the main body can be found in Appendix~\ref{app:convergence_proof}.

%% file: prelims.tex
\section{Preliminaries}
\label{sec:pre}

We use $\| \cdot \|$ to denote the vector 2-norm or matrix operator norm if not otherwise specified.
We use $\avg_{i=1}^n \cdots$ to denote $\frac{1}{n} \sum_{i=1}^n \cdots$.
\begin{definition}[Differential Privacy~\citep{DBLP:conf/eurocrypt/DworkKMMN06}]
    If for any two datasets $D$ and $D'$ which differ by a single data point and any $O\subseteq Range(\mathbb{A})$, an algorithm $\mathbb{A}$  satisfies 
    \[
        \P{\mathbb{A}(D)\in O} \leq e^\eps \cdot \P{\mathbb{A}(D')\in O} + \delta
        ,
    \]
    then algorithm $\mathbb{A}$ is said to be ($\eps, \delta)$-Differentially Private ($(\eps, \delta)$-DP).
\end{definition}


We want to minimize function $f(\vx) = \avg_{i=1}^N f_i(\vx)$, where each $f_i$ is a continuously differentiable function $\R^d \to \R$.  We make the following assumptions which are widely used for the analysis of high probability convergence of SGD~\cite{DBLP:journals/siamjo/NemirovskiJLS09, DBLP:journals/siamjo/GhadimiL13a, fang_SPIDERNearOptimal_2018, DBLP:conf/colt/HarveyLPR19, feldman2020private,DBLP:conf/iclr/FangLF023}.\footnote{
    Assumption~\ref{ass:main}.\ref{ass:stoc} is usually stated as $\Exp{\SqrNorm{g(\vx) - \nabla f(\vx)}} \le \sigma^2$, where $g$ is the stochastic gradient oracle.
    Assumptions~\ref{ass:main}.1 - \ref{ass:main}.3 are standard in non-convex settings for convergence to a first-order stationary point, while additional Assumption~\ref{ass:main}.\ref{ass:lip_hes} is standard for convergence to a second-order stationary point.
}
%
\begin{assumption}
    \label{ass:main}
    The objective function $f$ and the component functions $f_i$ satisfy the following conditions.
    \begin{enumerate}
        \item (Boundedness) 
            There exists $\fmax$ such that $\fmax > f(\x{0}) - f(\vx^*)$.
        \item\label{ass:stoc} (Stochastic variance) The stochastic variance is bounded by $\sigma^2$, i.e. $\avg_{i=1}^N \norm{\nabla f_i(\vx) - \nabla f(\vx)}^2 \le \sigma^2$.
        \item (Smoothness) The objective function $f$ is $L$-smooth, i.e. $\norm{\nabla f(\vx) - \nabla f(\vy)} \le L \norm{\vx - \vy}$ for all $\vx, \vy \in \R^d$.
        \item\label{ass:lip_hes} (Lipschitz Hessian) The objective function $f$ has a $\rho$-Lipschitz Hessian, i.e. $\norm{\nabla^2 f(\vx) - \nabla^2 f(\vy)} \le \rho \norm{\vx - \vy}$ for all $\vx, \vy \in \R^d$.
    \end{enumerate}
\end{assumption}
For a minibatch $\batch{t}$, we define $f_{\batch{t}} = \avg_{i \in \batch{t}} f_i$.
\begin{definition}
    \label{def:norm_subgaussian}
    A random vector $X$ is norm-subgaussian (or $\text{nSG}(\sigma^2)$) if $\Exp{\exp\pars{\frac{\SqrNorm{X - \Exp{X}}}{\sigma^2}}} \le e$.
\end{definition}
We make the following assumption, which is used to obtain high-probability bounds.
\begin{assumption}[Gradient concentration]
    \label{ass:grad_concentration}
    For any $\vx$, the stochastic gradient at $\vx$ is $\text{nSG}(\sigma^2)$:
    \[
        \avg\nolimits_{i=1}^N \exp \pars{\frac{\norm{\nabla f_i(\vx) - \nabla f(\vx)}^2}{\sigma^2}} < e
    \]
\end{assumption}
The above assumption is required in \Cref{lem:grad_bound} to get logarithmic dependence on $\delta$, but is not required in our convergence analysis if one is willing to add sufficiently large noise.

%

An approximate first-order stationary point requires that the gradient is small.
\begin{definition}[$\alpha$-FOSP]
    \label{def:fosp}
    We say that $\vx$ is an $\alpha$-first order stationary point ($\alpha$-FOSP) for a differentiable function $f$ if $\|\nabla f(\vx)\| < \alpha$.
\end{definition}
A second-order stationary point additionally requires that the function is almost convex near the point.
\begin{definition}[$\alpha$-SOSP~\citep{nesterovP06}]
    \label{def:sosp}
    We say that $\vx$ is an $\alpha$-second order stationary point ($\alpha$-SOSP) for a twice-differentiable function $f$ if $\|\nabla f(\vx)\| < \alpha$ and $\lmin(\nabla^2 f(\vx)) > -\sqrtra$.
\end{definition}

%% file: algorithm.tex
\section{Algorithm and Analysis}
\label{sec:algorithm}
Algorithm~\ref{alg:old_dpsgd} shows standard DP-SGD with clipping~\citep{abadi_DeepLearning_2016}.
Algorithm~\ref{alg:dpsgd} shows our simpler version of DP-SGD without clipping which, as we show next, still guarantees second-order convergence and differential privacy even without Lipschitz assumption.

\begin{algorithm}[t!]
    \caption{Differentially private Clipped SGD}
    \label{alg:old_dpsgd}
\begin{algorithmic}
    \STATE {\bfseries Input:} Clipping threshold $C$, noise variance $\Delta^2$, batch size $B$, number of iterations $T$
    \FOR{$t = 0, 1, 2, \ldots, T-1$}
        \STATE Sample minibatch $\batch{t}$ of size $B$ \\
        \STATE $\g{t}^{(i)} \gets \nabla f_i(\x{t}) / \max \pars{1, \frac{\|\nabla f_i(\x{t})\|}{C}}$ for each $i \in \batch{t}$  \\
        \STATE Sample $\xi_t \sim \mathcal{N}(0, \Delta^2 \id)$ \\
        \STATE $\x{t + 1} \gets \x{t} - \step \pars{\avg_{i \in \batch{t}} \g{t}^{(i)} + \xi_t}$
    \ENDFOR
\end{algorithmic}
\end{algorithm}

\begin{algorithm}[t!]
    \caption{Differentially private SGD without clipping } 
    \label{alg:dpsgd}
\begin{algorithmic}
    \STATE {\bfseries Input:} Batch size $B$, number of iterations $T$
    \STATE Choose noise scale $\Delta$ as in Theorem~\ref{thm:dp}
    \FOR{$t = 0, 1, 2, \ldots, T-1$}
        \STATE Uniformly sample minibatch $\batch{t}$ of size $B$ \\
        \STATE Sample $\xi_t \sim \mathcal{N}(0, \Delta^2 \id)$ \\
        \STATE $\x{t+1} \gets \x{t} - \step \pars{\avg_{i \in \batch{t}} \nabla f_i(\x{t}) + \xi_t}$
    \ENDFOR
\end{algorithmic}
\end{algorithm}



We begin in Section~\ref{sec:DP} by showing that Algorithm~\ref{alg:dpsgd} is indeed differentially private.
The proof of second-order convergence is more involved, so we give an outline of our analysis and some of the proofs in Section~\ref{ssec:convergence} (the details can be found in the Appendix).
We then combine the privacy and convergence analysis to give our main overall result in Section~\ref{ssec:combine}.
In Section~\ref{ssec:high_prob} we describe how to find a second-order stationary point w.h.p.
In this section, $\tOh{\cdot}$ hides polynomial dependence on $L, \rho, \fmax$, and polylogarithmic dependence on all parameters.

\subsection{Analysis of Differential Privacy} \label{sec:DP}

The proof that \cref{alg:dpsgd} is differentially private is quite intuitive: we use the fact from previous work that \cref{alg:old_dpsgd} is differentially private, and then show that, thanks to the way we set the noise in \cref{alg:dpsgd}, clipping is unlikely to happen and thus the two algorithms behave identically.  So the only difference in privacy is the small probability that clipping \emph{does} occur, which can be folded into the $\delta$ parameter of the privacy guarantee.  

\begin{theorem}[Standard DP for clipped SGD~\citep{abadi_DeepLearning_2016}]
    \label{thm:standard_dp}
    There exist constants $c_1$ and $c_2$, such that for any $\eps < c_1 T B^2 / n^2$, Algorithm~\ref{alg:old_dpsgd} is $(\eps, \delta)$-DP for any $\delta > 0$ if $\Delta > c_2 \frac{C \sqrt{T \log \nicefrac{1}{\delta}}}{n \eps}$.
\end{theorem}

We will use the following bound on the gradients from~\citet{DBLP:conf/iclr/FangLF023}, who proved first-order convergence for DP-SGD without Lipschitz assumption.

\begin{lemma}[Bound on the gradient; Proposition 4.3 from \citet{DBLP:conf/iclr/FangLF023}, simplified]
    \label{lem:grad_bound}

    Under Assumptions~\ref{ass:main} and~\ref{ass:grad_concentration}, when $\eta \le \min(1/L, 1 / (\sigma \sqrt{T}))$, with probability at least $1 - \delta$, for some absolute constant $c$:
    \begin{align*}
        &\max_{i \in [n], t = 0, ..., T-1} \|\nabla f_i(\x{t})\|
        \le 2 \sqrt{L \fmax} + \\
        &\qquad+ c \pars{L \sqrt{\log \nicefrac{T}{\delta}}
            + \sigma \sqrt{\log (n T)} + \sqrt{\sigma \log \nicefrac{1}{\delta}}
        }
        .
    \end{align*}
\end{lemma}

We can now prove our main differential privacy bound.
\begin{theorem}
    \label{thm:dp}
     Let $\eta \le \min(1/L, 1 / (\sigma \sqrt{T}))$.
     There exist $c_1, c_2 > 0$, such that for any $\delta > 0, \eps < c_1 T B^2 / n^2$, Algorithm~\ref{alg:dpsgd} is $(\eps, 2\delta)$-DP if $\Delta > c_2 \frac{C \sqrt{T \log \nicefrac{1}{\delta}}}{n \eps}$, where \\ $C = 2 \sqrt{L \fmax} + c \pars{L \sqrt{\log \nicefrac{T}{\delta}} + \sigma \sqrt{\log (n T)} + \sqrt{\sigma \log \nicefrac{1}{\delta}}}$.
\end{theorem}
\begin{proof}
    %
    From Theorem~\ref{thm:standard_dp}, we know that Algorithm~\ref{alg:old_dpsgd} is $(\eps,\delta)$-DP for our choice of $\Delta$.
    For our Algorithm~\ref{alg:dpsgd}, let $\mathcal{E}$ be the event that all gradient norms are less than $C$.
    By Lemma~\ref{lem:grad_bound}, all gradient norms are less than $C$ with probability at least $1 - \delta$, so $\mathcal{E}$ holds with high probability $1-\delta$.

    So by \cref{thm:standard_dp}, if we condition on $\mathcal{E}$, then \cref{alg:dpsgd} is $(\epsilon, \delta)$-private.
    On the other hand, if $\mathcal{E}$ doesn't hold, then we have no privacy guarantee, but this happens with probability at most $\delta$.
    Hence, the overall privacy guarantee is $(\epsilon, 2\delta)$ as claimed.
  %
    %
\end{proof}

\subsection{Outline of the Convergence Analysis}
\label{ssec:convergence}

In this section, we show that, for a fixed $\Delta^2$, for any $\alpha$, and for a certain choice of $\step$ and $T$, Algorithm~\ref{alg:dpsgd} finds an $\alpha$-SOSP.
In Section~\ref{ssec:combine}, we show for which $\alpha$ we guarantee both differential privacy and convergence.
Our main result shows that at least a constant fraction of iterates are $\alpha$-SOSP.
\begin{theorem}
    \label{thm:convergence_main}
    Under Assumption~\ref{ass:main}, when $\Delta^2 \ge \sigma^2 / B$ and $\step = \tOh{\alpha^2 / d \Delta^2}$
    for $T = 64 \fmax / (\step \alpha^2)$, w.h.p. at least half of visited iterates of \Cref{alg:dpsgd} are $\alpha$-SOSP.
\end{theorem}
%
We would like to point out that the fraction $1/2$ is somewhat arbitrary; we can guarantee that any constant fraction of the iterates are $\alpha$-SOSP, at the cost of increased number of iterations.
The full proof of this theorem can be found in Appendix~\ref{app:convergence_proof}; we outline the proof here.

The proof of convergence follows the general structure of that of~\citet{jin_NonconvexOptimization_2021}.
However, unlike the standard stochastic gradient settings, when we select the noise variance $\Delta^2$ guaranteeing the best convergence rate, for noisy SGD the variance must be chosen according to Theorem~\ref{thm:dp}.
Moreover, by construction, the noise is large compared with the gradients: its norm is proportional to $C \sqrt{d}$, where $C$ upper-bounds gradient norms.
Hence, our analysis must be able to handle the case when the noise is large.

In the large noise settings, we can assume smoothness of the objective function instead of smoothness of each component function.
\citet{jin_NonconvexOptimization_2021} require $\tilde{O}(\sigma^2 / \alpha^4)$ iterations to find an $\alpha$-SOSP when component function smoothness is assumed, and $\tilde{O}(d \sigma^2 / \alpha^4)$ iterations when only the objective function smoothness is assumed.
In our case, due to large noise, we don't lose the $d$ factor in the convergence rate, allowing us to get the standard $\nicefrac{d^{1/4}}{\sqrt{n \eps}}$ bound on $\alpha$.

Similarly, while \Cref{ass:grad_concentration} is required in \Cref{lem:grad_bound}, contrary to \citet{jin_NonconvexOptimization_2021}, in \Cref{thm:convergence_main} the assumption is not necessary when the noise is sufficiently large.
See Appendix~\ref{app:no_nSG} for the details.


%
\begin{proof}[Proof sketch]
    We independently bound the number of points with gradient norm larger than $\alpha$ (Lemma~\ref{lem:quarter_large_grad}) and the number of points with Hessian having an eigenvalue less than $-\sqrtra$ (Lemma~\ref{lem:quarter_saddle}).
    In Appendix~\ref{app:convergence_proof}, we prove that there are at most $\nicefrac{T}{4}$ points of the first type, and at most $\nicefrac{T}{4}$ points of the second type, which implies that at least $\nicefrac{T}{2}$ points are $\alpha$-SOSP.

    The main idea is to show that, if a point is not an $\alpha$-SOSP, then w.h.p. the objective value decreases by at least $\df = \tOm{\alpha^{3/2}}$ after at most $\escIter = \tOh{\nicefrac{1}{\step \sqrt{\alpha}}}$ iterations.
    It implies that w.h.p., if more than half of the points are not $\alpha$-SOSP, after $T$ iterations the objective value decreases by at least $\frac{\df T}{2\escIter} = \tOm{\frac{\alpha^{3/2} \fmax / (\step \alpha^2)}{1 / (\step \sqrt{\alpha})}} = \tOm{\fmax}$.
    We choose our parameters so that this value is strictly greater than $\fmax$, which leads to contradiction.
    
    \paragraph{Large gradients} For the points with gradients larger than $\alpha$, the analysis is similar to standard SGD analysis.
    We show that the function value decreases by $\df$ after one iteration w.h.p.
    By the folklore Descent Lemma, we have:
    \begin{align*}
        f(\x{t}) - f(\x{t+1})
        &\ge \InnerProd{\nabla f(\x{t}), \x{t} - \x{t+1}}
        - \frac{L}{2} \SqrNorm{\x{t} - \x{t+1}}
    \end{align*}
    Using that $\Exp{\x{t} - \x{t+1}} = \step \nabla f(\x{t})$ and $\ExpSqrNorm{\x{t} - \x{t+1}} = \step^2 (\SqrNorm{\nabla f(\x{t})} + \tsigma^2)$, where $\tsigma^2 = d \Delta^2 + \nicefrac{\sigma^2}{B}$ is the full noise variance, we have
    \begin{align*}
        \E[f(\x{t}) - f(\x{t+1})] &\ge \pars{\step - L \step^2} \|\nabla f(\x{t})\|^2 - \frac{L \step^2}{2} \tsigma^2
    \end{align*}
    Hence, for $\step \le \min(1, \nicefrac{\alpha^2}{\tsigma^2}) / (2L)$, if $\|\nabla f(\x{t})\| \ge \alpha$, the function decreases by $\step \alpha^2 / 4$ in expectation.
    In Lemma~\ref{lem:descent} we prove a high-probability version of this result.
    
    \paragraph{Saddle points}
    We first consider the case when the objective function is quadratic, i.e. $f(\vx) = \vx^\top \vH \vx / 2$.
    Then the gradient step is
    \begin{align*}
        \x{t+1} = \x{t} - \step \nabla f(\x{t}) = (\id - \step \vH) \x{t}
        .
    \end{align*}
    Hence, gradient descent emulates power method for matrix $\id - \step \vH$.
    Since $L$ is the smoothness constant, it is an upper bound on eigenvalues of the Hessian $\vH$, and by choosing $\step < \nicefrac{1}{L}$ we guarantee that $\id - \step \vH$ is positive definite.
    Hence, the smallest eigenvalue $-\gamma$ of $\vH$ corresponds to the largest eigenvalue $1 + \step \gamma > 1$ of $\id - \step \vH$, with the same eigenvector $\vv_1$.
    Hence, we know that after $\escIter$ iterations, the projection of $\x{\escIter}$ on $\vv_1$ increases by a factor of $(1 + \step \gamma)^\escIter \approx \exp(\escIter \step \gamma)$.
    Conversely, all positive eigenvalues of $\vH$ correspond to the eigenvalues between $0$ and $1$ of $1 - \step \vH$, and hence the projection on their eigenvectors is vanishing.

    The above reasoning shows that gradient descent for quadratic functions escapes from saddle points.
    For general functions, we need to take into account stochastic gradients and quadratic approximation error.
    The former issue is addressed by selecting a sufficiently small step size $\step \le \nicefrac{\alpha^2}{L \tsigma^2}$ so that the stochastic noise is smoothed.
    The latter issue is addressed by ensuring that the objective function improves while iterates stay in a ball where the Hessian doesn't change significantly.
    Using that $\gamma \ge \sqrtra$, this motivates the choice of parameters $\df$ and $\escIter$:
    \begin{compactitem}
        \item
            Since the Hessian is $\rho$-Lipschitz and the smallest eigenvalue is at most $-\sqrtra$, by ensuring that we stay in a ball of radius $\escRad = \sqrt{\alpha / 4\rho} = \Theta(\sqrt{\alpha})$, we guarantee that the Hessian has an eigenvalue at most $-\sqrtra / 2$.
        \item
            In this ball, the best objective improvement $\df$ we can hope for is if we move in the direction of $\vv_1$, which is $\sqrtra \escRad^2 = \Omega(\alpha^{3/2})$.
        \item
            As in calculations above, to reach the boundary of the ball, we need $\exp(\escIter \step \sqrtra) = \Theta(\escRad)$, leading to $\escIter = \tOh{1 / \step \sqrt{\alpha}}$.
    \end{compactitem}
    More formally, we use the fact that, if the point moves by distance $\escRad$ within $\escIter$ iterations, then the objective value improves by $\Omega(\df)$ (Corollary~\ref{cor:i_or_l}).
    We give the complete proof in Appendix~\ref{app:convergence_proof}.
\end{proof}

\subsection{Differentially Private Escaping from Saddle Points}
\label{ssec:combine}

In this section, we prove our main result, which combines the second-order convergence and differential privacy.
Intuitively, in Section~\ref{ssec:convergence} we showed that for any fixed $\Delta^2$, for any $\alpha$ we can select $\step$ and $T$ which guarantee the second-order convergence.
However, $\Delta^2$ is not fixed and depends on $T$ (see Theorem~\ref{thm:dp}), which leads to circular dependency: the number of iterations $T$ grows as variance $\Delta^2$ grows, and $\Delta^2$ grows as $T$ grows.
Hence, it is not a priori clear that it is possible to guarantee both differential privacy and second-order convergence simultaneously.
Nevertheless, since the number of iterations also depends on $\alpha$, we show that for a sufficiently large (but still quite small) $\alpha$ it is possible to guarantee both differential privacy and convergence to an $\alpha$-SOSP.

\begin{theorem}
    \label{thm:combine_main}
    Let Assumptions~\ref{ass:main} and~\ref{ass:grad_concentration} be satisfied, and let $\Delta$ satisfy conditions in Theorem~\ref{thm:dp} as well as $\Delta^2 \ge \sigma^2 / B$.
    Then for $\step = \tOh{\alpha^2 / (d \Delta^2)}$, Algorithm~\ref{alg:dpsgd} is $(\eps, 2\delta)$-DP and after $T = 64 \fmax / (\step \alpha^2)$ iterations,
    w.h.p. at least half of the points are $\alpha$-SOSP for $\alpha = \tilde{\Omega} \pars{\frac{\fmax d^{1/4}}{\sqrt{n \eps}}}$.
\end{theorem}
\begin{proof}
    From Theorem~\ref{thm:dp} we know that, to guarantee that Algorithm~\ref{alg:dpsgd} is $(\eps,\delta)$-differentially private, we must select $\Delta > c_2 \frac{C \sqrt{T \log \nicefrac{1}{\delta}}}{n \eps}$, where $C$ is selected as in Theorem~\ref{thm:dp}.
    We need to choose $\alpha$ so that such $\Delta$ exists.
    Recall that 
    \[
        T = \frac{64 \fmax}{\step \alpha^2}
        = c \frac{L \fmax (\sigma^2/B + d \Delta^2)}{\alpha^4}
        .
    \]
    for some constant $c$.
    It suffices to choose $\alpha$ so that for some constant $c$
    \[
        \Delta > c \frac{C \sqrt{L \fmax (\sigma^2/B + d \Delta^2) \log \nicefrac{1}{\delta}}}{n \eps \alpha^2}
        ,
    \]
    Using $\Delta^2 \ge \sigma^2 / B$, eliminating $\Delta$ on both sides, and rearranging the terms,
    the inequality is satisfied if we choose $\alpha = \tOm{\frac{\fmax d^{1/4}}{\sqrt{n \eps}}}$.
\end{proof}

\subsection{Differentially Private Selection of SOSP}
\label{ssec:high_prob}

While Theorem~\ref{thm:combine_main} shows that a substantial fraction of points are SOSP\footnote{This is common in analysis of SGD for non-convex functions. For example, the folklore analysis bounds $\avg_{t=0}^{T-1} \SqrNorm{\nabla f(\x{t})}$, implying that most of the iterates have small gradients}, we still need to find one of these points.
The simplest approach is to sample random points and check that the point is a SOSP by computing the norm of the gradient and the eigenvalue of the Hessian.
To preserve privacy, we use \textsc{AboveThreshold} algorithm by~\citet{dwork_AlgorithmicFoundations_2014}.
By adding Laplacian noise $\text{Lap}\pars{\nicefrac{C}{n \eps}}$ to the gradient norm and Laplacian noise $\text{Lap}(\nicefrac{L}{n \eps})$ to the smallest eigenvalue of the Hessian, we guarantee privacy.
Since $\alpha$ depends on $n$ as $n^{-1/2}$, $\sqrt{\alpha}$ depends on $n$ as $n^{-1/4}$, and noise depends on $n$ as $n^{-1}$, for a sufficiently large $n$, namely $n = \tOm{(\eps\fmax^2 \sqrt{d})^{-1}}$ we approximate the gradient norm up to additive error $\nicefrac{\alpha}{2}$, and the eigenvalue up to additive error $\nicefrac{\sqrtra}{2}$.

We note that the smallest eigenvalue can be computed efficiently.
Many machine learning frameworks allow efficient computation of stochastic Hessian-vector products~-- that is, for any vector $\vv$, one can compute $\nabla^2 f_i(\vx) \vv$~-- using a constant number of back-propagation passes.
With this, we can find the smallest eigenvalue using stochastic variants of the power method, see e.g.~\citet{oja1985stochastic}.




%% file: experiments.tex
\section{Experiments}
\label{sec:experiments}

In this section, we give the experimental results.
We compare convergence of DP-SGD (Algorithm~\ref{alg:dpsgd}) for various choices of $\eps$ with convergence of SGD.
We train a neural network on CIFAR-10 and CIFAR-100 image datasets~\citep{krizhevsky_LearningMultiple_2009}, as well as CoLa corpus~\citep{warstadt_NeuralNetwork_2019}.
For CIFAR-10, we use a convolutional neural network (see Appendix~\ref{app:experimental_details} for the detailed outline of the experimental setup) and a ResNet-18 neural network.
For CIFAR-100, we use pretrained ResNet-18, and for CoLa dataset we use pretrained uncased base BERT model.
All experiments are performed on a single NVIDIA A100 GPU.

\begin{figure*}[t!]
    \centering
    \begin{subfigure}[t]{\subfigwidth\textwidth}
        \includegraphics[width=\textwidth]{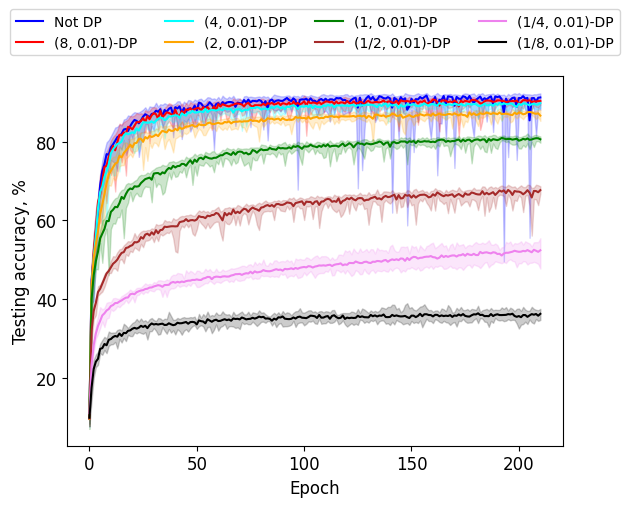}
        \caption{
            Convolutional model (see Appendix~\ref{app:experimental_details})
        }
        \label{fig:cifar10_cnn}
    \end{subfigure}
    \hfill
    \begin{subfigure}[t]{\subfigwidth\textwidth}
        \includegraphics[width=\textwidth]{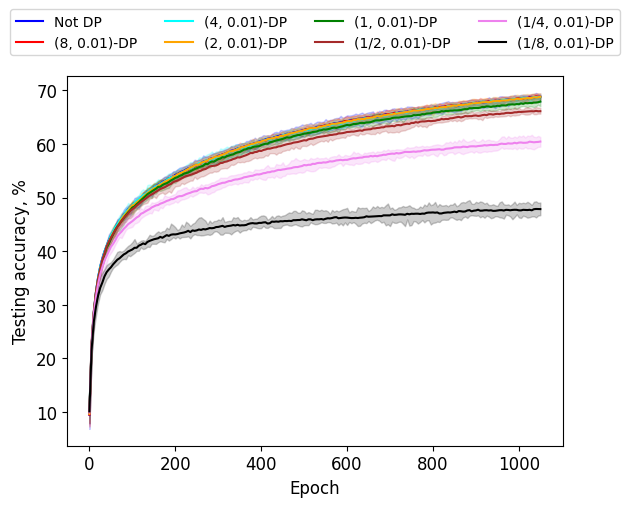}
        \caption{
            ResNet-18
        }
        \label{fig:cifar10_resnet}
    \end{subfigure}
    \caption{
        Testing accuracy of DP-SGD on CIFAR-10 dataset for various choices of $\eps$.
        Testing accuracy is averaged over $10$ runs, with the shaded area showing the minimum and the maximum values over the runs.
    }
    \label{fig:experiments_cifar10}
\end{figure*}

\begin{figure*}[t!]
    \centering
    \begin{subfigure}[t]{\subfigwidth\textwidth}
        \includegraphics[width=\textwidth]{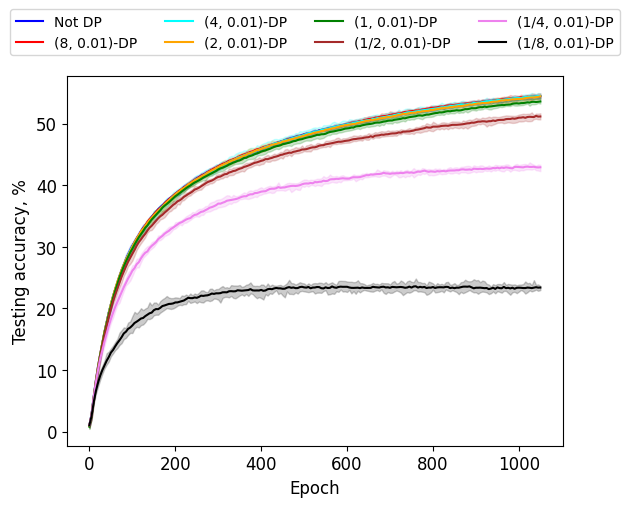}
        \caption{
            CIFAR-100 dataset
        }
        \label{fig:convergence}
    \end{subfigure}
    \hfill
    \begin{subfigure}[t]{\subfigwidth\textwidth}
        \includegraphics[width=\textwidth]{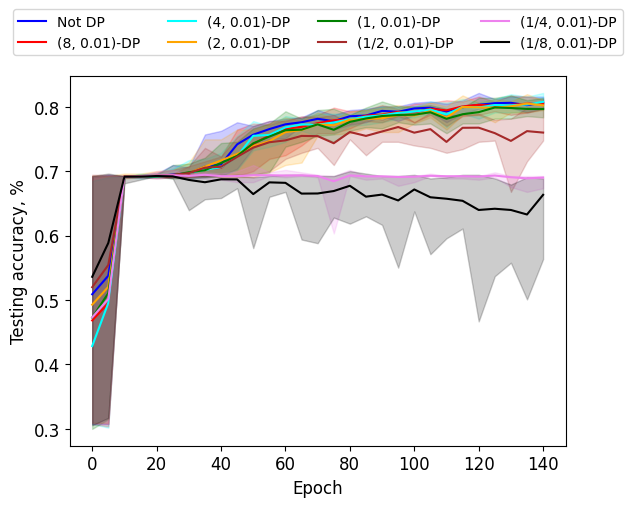}
        \caption{
            CoLa dataset
        }
        \label{fig:train_acc}
    \end{subfigure}
    \caption{
        Testing accuracy of DP-SGD  on CIFAR-100 and CoLa datasets.
        Testing accuracy is averaged over $10$ runs, with the shaded area showing the minimum and the maximum values.
    }
    \label{fig:experiments_cola_cifar100}
\end{figure*}

For our experiments, we choose error probability $\delta=0.01$ and learning rate $\step=0.1$.
We consider $\eps \in \{2, 4, 8\}$ similarly to~\citet{abadi_DeepLearning_2016}, as well as $\eps \in \{\nicefrac{1}{8}, \nicefrac{1}{4}, \nicefrac{1}{2}, 1\}$.
\Cref{fig:experiments_cifar10,fig:experiments_cifar10} shows testing accuracy for DP-SGD with the above choices of $\eps$, as well as testing accuracy of non-differentially private SGD.
Our results show monotone dependence on $\eps$, matching our prediction in Theorem~\ref{thm:combine_main} that for greater value of $\eps$ we find a better solution.
Moreover, for $\eps \in \{4, 8\}$, the testing accuracy is close to that of the non-differentially private SGD, which means that we obtain strong privacy guarantees without sacrificing the quality of the model.

Recall that the algorithm requires $C$~-- an upper bound on all gradient norms~-- which we choose empirically.
In Appendix~\ref{app:experimental_details} we present additional experimental results, including training accuracy and algorithm convergence when $C$ chosen as in \Cref{thm:dp}.


%% file: conclusion.tex
\section{Conclusion}
We have shown that DP-SGD, even without gradient clipping, converges to a second-order stationary point under minimal assumptions; in particular, without assuming that the loss function is Lipschitz.  As DP-SGD is the workhorse of modern private optimization, this fundamentally improves our understanding of private optimization. 
Moreover, we showed that experimentally, for the values of $\eps$ considered in~\citet{abadi_DeepLearning_2016}, on a standard CIFAR-10 dataset our performance is comparable to the \emph{non}-private classical SGD algorithm.


%% file: experimental_details.tex
\section{Experimental Details}
\label{app:experimental_details}

We perform experiments on NVIDIA RTX A5000.
To improve the performance of our experiments, we use FFCV library~\citep{leclerc2023ffcv}.
We use the Convolutional Neural Network with the following architecture:
\begin{lstlisting}
    conv_bn(3, 64, kernel_size=3, stride=1, padding=1),
    conv_bn(64, 128, kernel_size=5, stride=2, padding=2),
    Residual(Sequential(conv_bn(128, 128), conv_bn(128, 128))),
    conv_bn(128, 256, kernel_size=3, stride=1, padding=1),
    MaxPool2d(2),
    Residual(Sequential(conv_bn(256, 256), conv_bn(256, 256))),
    conv_bn(256, 128, kernel_size=3, stride=1, padding=0),
    AdaptiveMaxPool2d((1, 1)),
    Flatten(),
    Linear(128, 10, bias=False),
    Mul(0.2)    
\end{lstlisting}
where \lstinline{conv_bn}$(c_i, c_o, k, s, p)$ consists of 1) a \lstinline{Conv2d} layer with $c_i$ input channels, $c_o$ output channels, kernel size $k$, stride $s$, and padding $p$, 2) batch normalization layer, and 3) ReLU activation layer.

Aside from $\eps$ and $\delta$, our algorithm requires two hyperparameters: learning rate $\step$ and the Gaussian noise variance $\Delta^2$.
We select the learning rate to be $0.1$, which is the largest value which still achieves convergence.
The parameter $\Delta^2$ is computed per Theorem~\ref{thm:dp} based on $C$~-- the bound on the gradient norms throughout the algorithm execution.
While $C$ can be computed based on Lemma~\ref{lem:grad_bound}, this bound is pessimistic; in practice, we choose $C$ to be the upper bound on the gradient norm, which, for training a convolutional neural network on CIFAR-10 dataset, equals $2.5$ (as we verify in Figure~\ref{fig:grad_norms}).


In Figure~\ref{fig:testing_C}, we show testing accuracy, but with gradient bound $C$ chosen as in Theorem~\ref{thm:dp}.
The figure shows the trend similar to Figure~\ref{fig:convergence}, with lower accuracies due to using a pessimistic value of $C$.
The experimental setup is the same as in Section~\ref{sec:experiments}, i.e. on CIFAR-10 dataset with $\delta=0.01$.

\begin{figure}[t!]
    \centering
    \begin{subfigure}[t]{0.45\textwidth}
        \includegraphics[width=\textwidth]{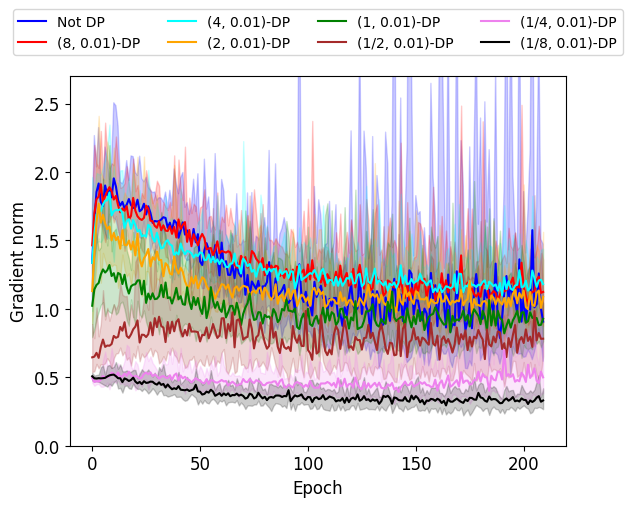}
        \caption{
            Training accuracy averaged over $10$ runs on CIFAR-10 dataset, with the shaded area showing the minimum and the maximum values over the runs.
        }
        \label{fig:grad_norms}
    \end{subfigure}
    \hfill
    \begin{subfigure}[t]{0.45\textwidth}
        \includegraphics[width=\textwidth]{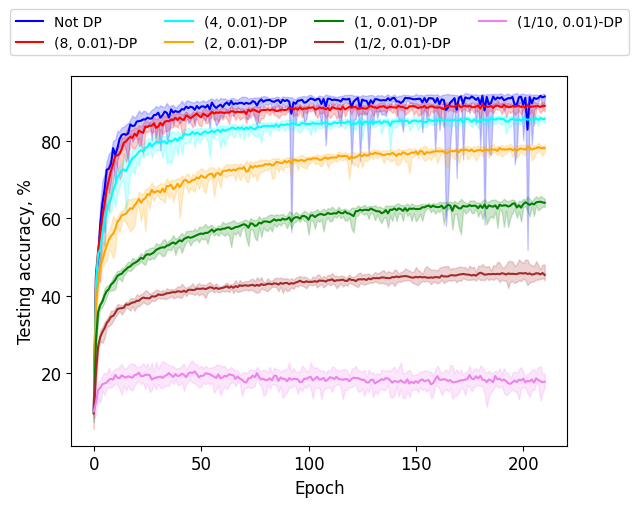}
        \caption{
            Testing accuracy averaged over $10$ runs on CIFAR-10 dataset, with the shaded area showing the minimum and the maximum values over the runs.
            Compared with Figure~\ref{fig:convergence}, we bound gradient norms using $C$ from Theorem~\ref{thm:dp}.
        }
        \label{fig:testing_C}
    \end{subfigure}
    \label{fig:additional_exp_details}
    \caption{Additional results for training a convolutional neural network using DP-SGD on CIFAR-10 dataset}
\end{figure}

%% file: proof_convergence.tex
\section{Convergence Proof}



\label{app:convergence_proof}


\begin{definition}
    \label{def:params}
    Let $\fullVar = \nicefrac{\sigma^2}{B} + d \Delta^2$ be total noise variance.
    Our choice of parameters is the following ($\cstep, \cit, \crad, \cf$ hide polylogarithmic dependence on all parameters):
    \begin{equation*}
        \begin{aligned}
            \text{Step size} && \step &= \cstep \min\pars{\frac{\alpha^2}{L (1 + \tsigma^2)}, \frac{\fmax}{L \tsigma^2}} \\
            \text{Total number of iterations} && T &= \frac{64\fmax}{\step \alpha^2} \\
            \text{Iterations required for escaping} && \escIter &=  \cit \frac {1} {\step \sqrtra} \\
            \text{Escaping radius} && \escRad &= \crad \sqrt {\alpha / \rho} \\
            \text{Objective decrease after successful escaping} && \df &= \cf \sqrtfacr
        \end{aligned}
    \end{equation*}
\end{definition}
Intuitively, the parameters are chosen as follows.
Assume that the current point is $\vzero$ and the function is quadratic, i.e. $f(\vx) = \vx^T H \vx$.
If $\vx$ is a saddle point, then we know that the smallest eigenvalue of $H$ is $-\eigen < -\sqrtra$ with the corresponding eigenvector $\vv_1$.
Then, the largest function decrease would be in the direction of $\vv_1$, and this direction can be found using power method for the appropriate matrix.
But, if we look at the gradient descent update:
\[
    \x{t+1}
    = \x{t} - \step \nabla f(\x{t})
    = \x{t} - \step 2 H \x{t}
    = (\id - \step 2 H) \x{t},
\]
we see that for quadratic functions, gradient descent is exactly the power method for the matrix $\id - \step 2H$, which has the same eigenvectors as $H$. The eigenvalue $-\eigen$ of $H$ corresponds to the eigenvalue $1 + 2 \step \eigen$ of $1 - \step 2 H$, which is the largest (by absolute value) eigenvalue when $\step \le 1/L$.
Note that power method requires some initial randomness, which explains while randomness is required to escape from saddle points.

The above intuition motivates the choice of parameters.
\begin{itemize}
    \item
        \textbf{Escaping radius $\escRad$}. Since in general the function is not necessarily quadratic, we want to guarantee that the Hessian doesn't change excessively.
        More specifically, we want the smallest eigenvalue to be at most $-\eigen / 2$ for the entire escaping procedure.
        By the Hessian-Lipschitz property $\|\nabla^2 f(\vx) - \nabla^2 f(\vy)\| \le \rho \|\vx - \vy\|$, if we work in the radius $\escRad = \eigen /(2\rho)$, then for any point within the radius $\escRad$ of the initial point, the smallest eigenvalue is at most $-\eigen + \eigen / 2 = -\eigen / 2$.
        When $\eigen > \sqrtra$, we can lower-bound $\escRad$ as $\sqrtra / (2\rho) = \sqrt{\alpha / (4\rho)}$.
    \item
        \textbf{Number of iterations required for escaping $\escIter$}.
        After $t$ iterations, the length of $\x{t}$ along the $\vv_1$ direction is $(1 + \step \eigen)^t \approx e^{t / (\step \eigen)}$.
        Hence, we reach the boundary of the ball of radius $\escRad$ after $(t \log \escRad) / (\step \eigen)$ iterations.
    \item
        \textbf{Function change after escaping $\df$}.
        Since we are working inside the ball of size $\escRad$, the best objective improvement $\df$ we can aim for is achievable by moving by $\escRad \vv_1$.
        This improves the objective by approximately $\escRad^2 \eigen$, which, by the choice of the parameters can be lower-bounded as $\alpha / \rho \cdot \sqrtra = \sqrt{\alpha^3 / \rho}$.
    \item
        \textbf{Total number of iterations $T$}. Every $\escIter$ iterations we improve objective by $\df$.
        Since the total objective improvement cannot exceed $\fmax$, the maximum number of iterations is
        \[
            \frac{\fmax \escIter}{\df}
            \le \frac{\fmax}{(\step \sqrtra)\sqrt{\alpha^3 / \rho}}
            = \frac{\fmax}{\step \alpha^2}
        \]
\end{itemize}

\begin{lemma}[Descent Lemma]
    \label{lem:descent}
    Assume that $\step < 1/L$, for any $t_0$ and $t$.
    Then, with probability at least $1 - \tdelta$, for some constant $c$:
    \[
        f(\x{t_0}) - f(\x{t_0 + t})
        \ge \frac{\step}{8} \sum_{\tau=t_0}^{t_0 + t - 1} \SqrNorm{\nabla f(\x{\tau})} - c \step \tsigma^2 (\step L t + \logErr)
        .
    \]
\end{lemma}

\begin{proof}
    W.l.o.g., we consider $t_0 = 0$.
    Let $\sgdNoise{\tau} = \nabla f_{\batch{t}}(\x{\tau}) - \nabla f(\x{\tau})$ be the stochastic noise at iteration $\tau$.
    By the folklore Descent Lemma,
    \begin{align*}
        f(\x{\tau}) - f(\x{\tau+1})
        &\ge \step \InnerProd{\nabla f(\x{\tau}), \nabla f_{\batch{t}}(\x{\tau})}
            - \frac{L}{2} \SqrNorm{\x{\tau + 1} - \x{\tau}} \\
        &= \step \InnerProd{\nabla f(\x{\tau}), \nabla f(\x{\tau}) + \sgdNoise{\tau} + \ourNoise{\tau}}
            - \frac{L\step^2}{2} \SqrNorm{\nabla f(\x{\tau}) + \sgdNoise{\tau} + \ourNoise{\tau}} \\
        &\ge \step \SqrNorm{\nabla f(\x{\tau})} + \step \InnerProd{\nabla f(\x{\tau}), \sgdNoise{\tau} + \ourNoise{\tau}}
            - \frac{\step^2 L}{2} \pars{\frac{3}{2}\SqrNorm{\nabla f(\x{\tau})} + 3 \SqrNorm{\sgdNoise{\tau} + \ourNoise{\tau}}}
        ,
    \end{align*}
    where we used inequality $\SqrNorm{a + b} \le (1 + \nu) \SqrNorm{a} + (1 + 1/\nu) \SqrNorm{b}$ for $\nu = 1/2$.
    By telescoping, we have:
    \begin{align*}
        f(\x{0}) - f(\x{t})
        &\ge \frac{\step}{4} \sum_{\tau=0}^{t-1} \SqrNorm{\nabla f(\x{\tau})}
            + \step \sum_{\tau=0}^{t-1} \InnerProd{\nabla f(\x{\tau}), \sgdNoise{\tau} + \ourNoise{\tau}}
            - \frac{3 \step^2 L}{2} \sum_{\tau=0}^{t-1} \SqrNorm{\sgdNoise{\tau} + \ourNoise{\tau}}
    \end{align*}
    Using a concentration inequality for nSG variables~\citep{jin_ShortNote_2019}, with probability at least $1 - \tdelta$ for some constant $c$ we have
    \[
        \abs{\step \sum_{\tau=0}^{t-1} \InnerProd{\nabla f(\x{\tau}), \sgdNoise{\tau} + \ourNoise{\tau}}}
        \le \frac{\step}{8} \sum_{\tau=0}^{t-1} \SqrNorm{\nabla f(\x{\tau})} + c \step \fullVar^2 \logErr
    \]
    For the last term, with probability at least $1 - \tdelta$ we have
    \[
        \frac{3 \step^2 L}{2} \sum_{\tau=0}^{t-1} \SqrNorm{\sgdNoise{\tau} + \ourNoise{\tau}}
        \le c \step^2 L \fullVar^2 (t + \logErr).
    \]
    Combining the above inequalities and using $\step \le 1/L$ finishes the proof.
\end{proof}

\subsection{Large Gradient Case}

The first application of the Descent Lemma is to bound the number of iterations when the gradient is large.
\begin{lemma}
    \label{lem:quarter_large_grad}
    When $T \ge \frac{64 \fmax}{\step \alpha^2}$ and other parameters are as in \Cref{def:params}, the number of iterations $t$ with $\|\nabla f(\x{t})\| \ge \alpha$ is at most $\nicefrac T4$.
\end{lemma}

\begin{proof}
    For the sake of contradiction, assume that the number of points $\x{t}$ with $\|\nabla f(\x{t})\| \ge \alpha$ is greater than $\nicefrac T4$.
    Then, by the Descent Lemma we have
    \begin{align*}
        f(\x{0}) - f(\x{T})
        &\ge \frac{\step}{8} \sum_{\tau=0}^{T - 1} \SqrNorm{\nabla f(\x{\tau})}
            - c \step \tsigma^2 (\step L T + \logErr) \\
        &\ge \frac{\step T \alpha^2}{32}  - c \step T \cdot \tsigma^2 \pars{\step L + \frac{\logErr}{T}}
    \end{align*}
    We will show that the right-hand side is at least $\fmax$, leading to contradiction.
    For that, it suffices to guarantee that $\frac{\step T \alpha^2}{32} \ge 2\fmax$ and $c \step T \cdot \tsigma^2 \pars{\step L + \frac{\logErr}{T}} \le \fmax$.
    The former condition is satisfied by our choice of $T$, and substituting the bound on $\fmax$ into the latter condition shows that it suffices to prove
    \[
        64 c \tsigma^2 \pars{\step L + \frac{\step \alpha^2 \logErr}{64 \fmax}} \le \alpha^2
    \]
    It suffices to bound each term with $\alpha^2 / 2$, which is satisfied for
    \begin{align*}
        \step &\le \frac{\alpha^2}{128 \cdot c \tsigma^2 L} \\
        \step &\le \frac{\fmax\logErr}{2 \cdot c \tsigma^2 L}
    \end{align*}
    Both of these conditions are satisfied by our choice of $\step$, with the appropriate choice of $\cstep$.
\end{proof}

\subsection{Escaping From Saddle Points}

The Descent Lemma also implies the following result, which upper-bounds the distance change in terms of the objective changes.
Its \Cref{cor:i_or_l} states that, if after a bounded number of iterations the point moves far from the original point, then the objective value sufficiently decreases.
\begin{lemma}[Improve or localize]
    \label{lem:improve_or_localize}
    For any $t_0$ and $t$, with probability at least $1 - \tdelta$, for some constant $c$:
    \begin{align*}
        \SqrNorm{\x{t_0 + t} - \x{t_0}}
        \le c \step t \pars{(f(\x{t_0}) - f(\x{t_0 + t}) + \step \tsigma^2 (\step L t + \logErr)}
    \end{align*}
\end{lemma}
\begin{proof}
    W.l.o.g. we consider the case when $t_0 = 0$.
    By the definition of $\x{\tau + 1}$, we have
    \begin{align*}
        \x{\tau} - \x{\tau + 1}
        = \step (\nabla f_{\batch{\tau}}(\x{\tau}) + \ourNoise{\tau})
        = \step (\nabla f(\x{\tau}) + \sgdNoise{\tau} + \ourNoise{\tau}),
    \end{align*}
    where $\sgdNoise{\tau} = \nabla f_{\batch{\tau}}(\x{\tau}) - \nabla f(\x{\tau})$ is stochastic gradient noise.
    By telescoping, we have:
    \begin{align*}
        \|\x{0} - \x{t}\|
        &= \step \norm{\sum_{\tau = 0}^{t - 1}(\nabla f(\x{\tau}) + \sgdNoise{\tau} + \ourNoise{\tau})}
    \end{align*}
    By Cauchy-Schwarz Inequality, we have:
    \begin{align*}
        \SqrNorm{\x{0} - \x{t}}
        &\le 2 \step^2 \pars{\SqrNorm{\sum_{\tau = 0}^{t - 1}\nabla f(\x{\tau})}
                            + \SqrNorm{\sum_{\tau = 0}^{t - 1} (\sgdNoise{\tau} + \ourNoise{\tau})}} \\
        &\le 2 \step^2 \pars{t \sum_{\tau = 0}^{t - 1}\SqrNorm{\nabla f(\x{\tau})}
                            + \SqrNorm{\sum_{\tau = 0}^{t - 1} (\sgdNoise{\tau} + \ourNoise{\tau})}}
    \end{align*}
    Since $\sgdNoise{\tau} + \ourNoise{\tau}$ is $\mathrm{nSG}(\fullVar^2)$, by applying concentration inequality for $\mathrm{nSG}$ variables, we have with probability at least $1 - \tdelta$:
    \begin{align*}
        \SqrNorm{\x{0} - \x{t}}
        &\le 2 \step^2 \pars{t \sum_{\tau = 0}^{t - 1}\SqrNorm{\nabla f(\x{\tau})} + c t \fullVar^2 \logErr}
    \end{align*}
    Finally, applying the Descent Lemma, we get
    \begin{align*}
        \SqrNorm{\x{0} - \x{t}}
        &\le \step c t \pars{f(\x{0}) - f(\x{t}) + \step \fullVar^2 (\step L t + \logErr)}
    \end{align*}
    for some constant $c$.
\end{proof}

\begin{corollary}
    \label{cor:i_or_l}
    For the parameters chosen as in \Cref{def:params}, if for some $t \le \escIter$ we have $\SqrNorm{\x{t_0 + \tau} - \x{t_0}} \ge \escRad$, then $f(\x{t_0}) - f(\x{t_0 + \tau}) \ge \df$ with probability at least $1 - \tdelta$.
\end{corollary}
\begin{proof}
    From \Cref{lem:improve_or_localize}, we know that for some $c$:
    \begin{align*}
        f(\x{t_0}) - f(\x{t_0 + t})
        &\ge \frac{\escRad^2}{c \step \escIter} - \step \tsigma^2 (\step L \escIter + \logErr) \\
        &\ge \frac{\crad^2}{c \cit} \sqrtfacr - \step \tsigma^2 \pars{\frac{\cit L}{\sqrtra} + \logErr}
    \end{align*}
    We want to guarantee that the right-hand side is at most $\df = \cf \sqrtfacr$.
    It suffices to guarantee that
    \begin{align*}
        \crad^2 / (c \cit) \ge 2 \cf \qquad \text{and} \qquad \step \tsigma^2 (\cit L / \sqrtra + \logErr) \le \cf \sqrtfacr.
    \end{align*}
    The former holds for the appropriate choice of $\crad$, $\cit$, and $\cf$, and the latter holds by choosing
    \[
        \step \le (\alpha^2 / L \tsigma^2) \cdot  \cf / \pars{\cit + \logErr},
    \]
    where we assumed $L \ge \sqrtra$, since otherwise any $\alpha$-FOSP is an $\alpha$-SOSP because $L$ is an upper bound on the absolute values of the Hessian.
\end{proof}

\begin{definition}[Coupling sequences]
    \label{def:coupling_sequence}
    Let $\vv_1$ be the eigenvector corresponding to the smallest eigenvalue of $\nabla^2 f(\x{t_0})$.
    The coupling sequences starting from $\x{t_0}$ are defined as follows:
    \begin{equation}
        \begin{aligned}
            \ourNoise{t_0+t} &\sim \mathcal{N}(0, \fullVar^2 \id) \\
                \tourNoise{t_0+t} &= \ourNoise{t_0+t} - 2 \InnerProd{ \vv_1, \ourNoise{t_0+t} } \vv_1  \\
            \x{t_0 + t+1} &= \x{t_0+t} - \step (\nabla f_{\batch{t}}(\x{t_0 + t}) + \ourNoise{t_0 + t}) \\
            \tx{t_0 + t+1} &= \tx{t_0+t} - \step (\nabla f_{\batch{t}}(\tx{t_0 + t}) + \tourNoise{t_0 + t}) \\
            \sgdNoise{t_0 + t} &= \nabla f_{\batch{t}}(\x{t_0 + t}) - \nabla f(\x{t_0 + t}) \\
            \tsgdNoise{t_0 + t} &= \nabla f_{\batch{t}}(\tx{t_0 + t}) - \nabla f(\tx{t_0 + t}) \\
        \end{aligned}\nonumber
    \end{equation}
    Note that:
    \begin{itemize}
        \item the artificial noise $\tourNoise{t_0 + t}$ has the opposite sign in direction $\vv_1$ compared with $\ourNoise{t_0 + t}$,
        \item $\{\x{t_0 + t}\}_t$ and $\{\tx{t_0 + t}\}_t$ are sampled from the same distribution.
    \end{itemize}
    We next define the difference between the coupling sequences:
    \begin{equation}
        \begin{aligned}
            \difourNoise{t} &= \xi_{t_0+t} - \xi'_{t_0+t} 
                = 2 \InnerProd{ \vv_1, \xi_{t_0+t} } \vv_1  \\
            \difx{t_0 + t+1} &= \x{t_0+t} - \tx{t_0+t} \\
            \difsgdNoise{t_0 + t} &= \sgdNoise{t_0 + t} - \tsgdNoise{t_0 + t}
        \end{aligned}\nonumber
    \end{equation}
\end{definition}



From now on, we assume that we want to escape from point $\x{t_0}$ such that $\lmin(\nabla^2 f(\x{t_0})) < -\sqrtra$.
W.l.o.g. we assume $t_0 = 0$.
Let $H = \nabla^2 f(\x{0})$ be the Hessian of $f$ at point $\x{0}$, let $-\eigen = \lmin(H)$ be the smallest eigenvalue of $H$, and $\vv_1$ be the eigenvector corresponding to $-\eigen$.

\begin{lemma}
    For every $t$,
    \[
        \difx{t} = -\dif{t} - \sgdErr{t} -\quadErr{t},
    \]
    where
    \begin{align*}
        \dif{t} &= \step \sum_{\tau = 0}^{t - 1} (\id - \step H)^{t - 1 - \tau} \difourNoise{\tau} \\
        \sgdErr{t} &= \step \sum_{\tau = 0}^{t - 1} (\id - \step H)^{t - 1 - \tau} \difsgdNoise{\tau} \\
        \quadErr{t} &= \step \sum_{\tau = 0}^{t - 1} (\id - \step H)^{t - 1 - \tau} \hesInt{\tau} \difx{\tau},
            \quad \text{ where } \quad \hesInt{\tau} = \int_0^1  (\nabla^2 f(\psi \x{\tau} + (1 - \psi) \tx{\tau}) - H)\ d\psi
    \end{align*}
\end{lemma}
\begin{proof}
    Proof by induction.
    \begin{align*}
        \difx{t + 1}
        &= \difx{t} - \step (\difsgdNoise{t} + \difsgdNoise{t}) - \step (\nabla f(\x{t}) - \nabla f(\tx{t})) \\
        &= (1 - \step H) \difx{t} - \step (\difsgdNoise{t} + \difsgdNoise{t}) - \step (\nabla f(\x{t}) - \nabla f(\tx{t}) - H \difx{t})
    \end{align*}
    Using that $\nabla f(\x{t}) - \nabla f(\tx{t}) = \int_0^1  \nabla^2 f(\psi \x{\tau} + (1 - \psi) \tx{\tau}) \difx{\tau}\ d\psi$, by telescoping we complete the proof.
\end{proof}
In the following statements, the geometric series $\alpha_t^2 = \sum_{\tau = 0}^{t-1} (1 + \step \eigen)^{2 \tau}$ occurs frequently.
Intuitively, each iteration increases the distance approximately by a factor of $(1 + \step \eigen)$~-- the largest eigenvalue of matrix $I - \step H$.
The noise added at iteration $\tau$ gets amplified by $(1 + \step \eigen)^{t - \tau}$ by iteration $t$, and, if every iteration an independent noise of variance $\Delta^2$ is added, $\Delta^2 \alpha_t^2$ is the total variance of such a noise accumulated over the iterations.
\begin{lemma}
    \label{lem:alpha_beta}
    Let $\alpha_t = \sqrt{\sum_{\tau = 0}^{t-1} (1 + \step \eigen)^{2 \tau}}$ and $\beta_t = (1 + \step \eigen)^t / \sqrt{\step \eigen}$.
    If $\step \eigen < 1$, then $\alpha_t \le \beta_t$ for any $t$ and $\alpha_t \ge \beta_t / 3$ for $t \ge 2 / (\step \eigen)$.
\end{lemma}
\begin{proof}
    As a sum of geometric progression, we have
    \[
        \alpha_t^2
        = \frac{(1 + \step \eigen)^{2t} - 1}{(1 + \step \eigen)^2 - 1}
        = \frac{(1 + \step \eigen)^{2t} - 1}{\step \eigen (2 + \step \eigen)},
    \]
    and $\alpha_t^2 \le \beta_t^2$ follows immediately.
    On the other hand, for $t \ge 1 / (\step \eigen)$ we have $(1 + \step \eigen)^{2t} \ge 2$, and hence $(1 + \step \eigen)^{2t}  - 1 \ge (1 + \step \eigen)^{2t} / 2$.
    Finally, by our choice of the step size, we have $\step \eigen \le \step L \le 1$, and hence $2 + \step \eigen \le 3$.
    This implies $\alpha_t^2 \ge \beta_t^2 / 6$, finishing the proof.
\end{proof}

\begin{lemma}
    \label{lem:good_term}
    Let $\beta_t = (1 + \step \eigen)^t / \sqrt{\step \eigen}$.
    Then for any $t$ and $\tdelta$, for some constant $c$:
    \begin{align*}
        \P{\|\dif{t}\| \ge c \beta_t \step \Delta \sqrt{\logErr}} &\le \tdelta \\
        \P{\|\dif{t}\| \le \frac{\beta_t \step \Delta}{10}} &\le \frac 13
    \end{align*}
\end{lemma}

\begin{proof}
    Since $\difourNoise{\tau} = 2\InnerProd{\ourNoise{\tau}, \vv_1} \vv_1$ and $\vv_1$ is an eigenvalue of $H$ corresponding to eigenvalue $-\eigen$, we have
    \[
        \dif{t}
        = \step \sum_{\tau = 0}^{t - 1} (\id - \step H)^{t - 1 - \tau} \difourNoise{\tau}
        = \pars{\step \sum_{\tau = 0}^{t - 1} (1 + \step \eigen)^{t - 1 - \tau} \InnerProd{\ourNoise{\tau}, \vv_1}} \vv_1
    \]
    Hence, $\dif{t}$ is a random variable which is parallel to $\vv_1$.
    Since each $\ourNoise{t}$ is a Gaussian random variable with the covariance matrix $\Delta^2 \id$, each $\InnerProd{\ourNoise{t}, \vv_1}$ a Gaussian random variable with variance $\Delta^2$.
    Hence, $\dif{t}$ is also a Gaussian random variable, and
    \[
        \Var{\|\dif{t}\|}
        = \step \sum_{\tau = 0}^{t - 1} (1 + \step \eigen)^{\tau} \Delta^2
        = \Delta^2 \alpha_t^2
    \]
    The upper bound on $\|\dif{t}\|$ follows from \Cref{lem:alpha_beta} and standard concentration inequality for Guassians.
    The lower bound follows from inequality $\P{|X| \le \sigma a} \le a$ for $X \sim \mathcal{N}(0, \sigma^2)$ and any $a$.
\end{proof}

\begin{lemma}
    \label{lem:error_terms}
    Let $\mathcal{E}_t$ be the event ``$\|\x{0} - \x{\tau}\| < \escRad$ and $\|\x{0} - \tx{\tau}\| < \escRad$ for all $\tau < t$''.
    Then, when $\Delta^2 \ge \sigma^2 / B$, for some constant $c$ we have
    \[
        \P{\mathcal{E}_\escIter
            \implies \max\nolimits_{\tau=0}^{\escIter - 1} \norm{\sgdErr{\tau} + \quadErr{\tau}}
            \le 
            c \beta_\tau \step \Delta \sqrt{\logErr}
        } \ge 1 - \escIter \tdelta
    \]
\end{lemma}
    
\begin{proof}
    By induction on $t$, we will prove that $\P{\mathcal{E}_t \implies \|\difx{t}\| \le c \step \Delta \beta_t} \le 1 - t\fullVar$ and
    \[
        \P{\mathcal{E}_t
            \implies \max\nolimits_{\tau=0}^{t - 1} \norm{\sgdErr{\tau} + \quadErr{\tau}}
            \le  c \beta_\tau \step \Delta \sqrt{\logErr}
        } \ge 1 - t \tdelta.
    \]
    Base case is $t=0$, in which case $\difx{0} = \vzero$.
    For the induction step, since $\difx{t} = -\dif{t} - \sgdErr{t} -\quadErr{t}$, we bound each term separately.
    \paragraph{Bounding $\dif{t}$}
    By \Cref{lem:good_term}, we have $\P{\|\dif{t}\| \ge c \beta_t \step \Delta \sqrt{\logErr}} \le \tdelta$.

    \paragraph{Bounding $\quadErr{t}$}
    Assuming $\mathcal{E}_t$, by the Hessian-Lipschitz property we have
    \begin{align*}
        \|\hesInt{\tau} \difx{\tau}\|
        &= \norm{\int_0^1  (\nabla^2 f(\psi \x{\tau} + (1 - \psi) \tx{\tau}) - H)\ d\psi \cdot \difx{\tau}} \\
        &\le 2 \rho \escRad \|\difx{\tau}\| \\
        &\le 2 \rho \escRad \cdot c \beta_\tau \step \Delta \sqrt{\logErr},
    \end{align*}
    and we bound $\quadErr{t}$ as
    \begin{align*}
        \|\quadErr{t}\|
        &= \norm{\step \sum_{\tau = 0}^{t - 1} (\id - \step H)^{t - 1 - \tau} \hesInt{\tau} \difx{\tau}} \\
        &\le \sum_{\tau = 0}^{t - 1} (1 + \step \gamma)^{t - 1 - \tau} 2 \rho \escRad \cdot c \beta_\tau \step \Delta \sqrt{\logErr} \\
    \end{align*}
    Since $(1 + \step \gamma)^{t - 1 - \tau} \beta_\tau = \beta_{t-1}$,
    we can bound $\|\quadErr{t}\|$ as 
    \begin{align*}
        \|\quadErr{t}\|
        &\le 2 \rho \escRad \cdot c \escIter \beta_t \step \Delta \sqrt{\logErr} \\
        &= 2 c \crad \cit \beta_t \Delta \sqrt{\logErr},
    \end{align*}
    where the last equality is due to our choice of parameters in \Cref{def:params}:
    \[
        \escRad \escIter
        = \crad \sqrt {\alpha / \rho} \cdot \cit \frac {1} {\step \sqrtra}
        = \frac{\crad \cit}{\rho \step}
    \]

    \paragraph{Bounding $\sgdErr{t}$}
    
    %
    By \Cref{ass:grad_concentration}, each $\sgdErr{t}$ is $\mathrm{nSG}(\sigma^2 / B)$.
    Using the concentration inequality for nSG variables, with probability at least $1 - \tdelta$ we have
    \begin{align*}
        \|\sgdErr{t}\|
        &= \norm{\step \sum_{\tau = 0}^{t - 1} (\id - \step H)^{t - 1 - \tau} \difsgdNoise{\tau}} \\
        &\le \step \sqrt{\sum_{\tau = 0}^{t - 1} (1 + \step \eigen)^{2(t - 1 - \tau)} \frac{\sigma^2}{B}} \cdot \sqrt{\logErr} \\
        &\le \step \sigma \beta_t \sqrt{\logErr} / \sqrt{B} \\
    \end{align*}
    Using assumption $\Delta^2 \ge \sigma^2 / B$, we get the desired bound.
\end{proof}

\begin{lemma}
    \label{lem:bucket_change}
    When $\Delta^2 \ge \sigma^2 / B$, with probability at least $1 - \escIter \tdelta$, we have
    \[
        f(\x{0}) - f(\x{\escIter}) > -\df / 100
    \]
    Moreover, with probability at least $1/3 - \escIter \tdelta$, we have:
    \[
        f(\x{0}) - f(\x{\escIter}) > \df
    \]
\end{lemma}

\begin{proof}
    The first claim follows from the descent lemma by our choice of parameters:
    \[
        f(\x{0}) - f(\x{\escIter})
        > -c \step \fullVar^2 (\step L \escIter + \logErr)
        > -\df / 10
    \]
    From Corollary~\ref{cor:i_or_l}, we know that with probability at least $1-\tdelta$, if $\|\x{t} - \x{0}\| > \escRad$ for some $t \le \escIter$, then $f(\x{0}) - f(\x{t}) > \df$ with probability at least $1 - \tdelta$.
    Since the $\tx{t}$ has the same distribution as $\x{t}$, the same holds for $\|\tx{t} - \x{0}\|$.
    Hence, if either $\|\x{t} - \x{0}\| > \escRad$ or $\|\tx{t} - \x{0}\| > \escRad$, then with probability $\frac{1}{2} - 2 \tdelta$ we have $f(\x{0}) - f(\x{t})$.
    
    For the rest of the proof, we can assume that $\|\x{t} - \x{0}\| < \escRad$ and $\|\tx{t} - \x{0}\| < \escRad$ for all $t \le \escIter$.
    By Lemmas~\ref{lem:good_term} and~\ref{lem:error_terms} we know that $\|\dif{\escIter}\| \ge \step \beta_t \Delta \sqrt{\logErr} / 10$ and $\|\sgdErr{t} + \quadErr{t}\| \le \step \beta_t \Delta \sqrt{\logErr} / 20$ with probability $1 - \tdelta$.
    Hence,
    \begin{align*}
        \max(\|\x{\escIter} - \x{0}\|, \|\tx{\escIter} - \x{0}\|)
        &\ge \frac{1}{2} \|\difx{\escIter}\| \\
        &\ge \frac 12 \pars{\|\dif{\escIter}\| - \|\sgdErr{\escIter} + \quadErr{t}\|} \\
        &\ge \frac{\step \beta_\escIter \Delta}{40} \sqrt{\logErr} \\
        &= (1 + \step \eigen)^\escIter \frac{1}{40}  \Delta \sqrt{\frac{\step}{\eigen}} \sqrt{\logErr} \\
        &\ge \escRad,
    \end{align*}
    where the last inequality follows by the appropriate choice of logarithmic term in $\cit$ (note that $(1 + \step \eigen)^\escIter$ dominates other terms).
    This contradicts the assumption $\max(\|\x{\escIter} - \x{0}\|, \|\tx{\escIter} - \x{0}\|) \le \escRad$.
\end{proof}

\begin{lemma}
    \label{lem:quarter_saddle}
    When $\Delta^2 \ge \sigma^2 / B$, for $T = 64 \fmax / (\step \alpha^2)$, with $\step$ chosen as in \Cref{def:params}, with probability at least $1 - T \tdelta$ we have $\lmin(\nabla^2 f(\x{t})) < -\sqrtra$ for at most $T/4$ points $\x{t}$.
\end{lemma}

\begin{proof}
    Proof by contradiction.
    Let's split all iterations $0,...,T-1$ into buckets of size $\escIter$.
    There are $T/\escIter = 64\fmax \sqrtra / \cit$ buckets, and at least quarter of the buckets have a point $\x{t}$ with $\lmin(\nabla^2 f(\x{t})) < -\sqrtra$ in it.
    Note that within each bucket, by Lemma~\ref{lem:bucket_change}, the objective increases by at most $\df / 100$.
    Out of such buckets, at least half of them are non-adjacent, and we consider points $\x{t}$ with $\lmin(\nabla^2 f(\x{t})) < -\sqrtra$ in these buckets.
    With probability at least $1/3 - \tdelta$, we have $f(\x{t}) - f(\x{t + \escIter}) > \df$ for each such $\x{t}$.
    Hence,
    \begin{align*}
        \fmax
        &\ge \frac{T \df}{10\cdot \escIter} \\
        &= \frac{1}{10} \cdot \frac{64\fmax}{\step \alpha^2} \cdot \cf \sqrtfacr \cdot \frac{\step \sqrtra}{\cit} \\
        &= \fmax \frac{64 \cdot \cf}{10 \cdot \cit}
    \end{align*}
    Hence, by selecting $\cit$ and $\cf$ so that $\frac{64 \cdot \cf}{10 \cdot \cit} > 1$, we arrive at contradiction.
\end{proof}



\begin{theorem}
    \label{thm:combine}
    When $\Delta^2 \ge \sigma^2 / B$, for $T = 64 \fmax / (\step \alpha^2)$, with probability at least $1 - \delta$, at least half of visited points are $\alpha$-SOSP.
\end{theorem}
\begin{proof}
    We choose $\tdelta = \delta / T$.
    Since dependence on $1/\tdelta$ is poly-logarithmic, this results in poly-logarithmic dependence on $1/\delta$ and $T$.

    From Lemma~\ref{lem:quarter_large_grad}, we know that the number of points with large gradient is at most $T/4$.
    From Lemma~\ref{lem:quarter_saddle} we know that the number of points $\x{t}$ with $\lmin(\nabla^2 f(\x{t})) < -\sqrtra$ is at most $T/4$.
    Hence, at least half of the points are $\alpha$-SOSP.
\end{proof}

%% file: proofs_no_nSG.tex
\section{Convergence Without Gradient Concentration Assumption}
\label{app:no_nSG}

In this section, we prove that DP-SGD finds a SOSP even without stochastic gradient concentration assumption (\Cref{ass:grad_concentration}).
We show that we recover the same convergence rate (w.r.t $\Delta$), although a substantially larger $\Delta$ might be required.
As shown in \Cref{thm:combine_main}, assuming that the bound on the norm of the gradients doesn't change, we find the $\alpha$-SOSP for the same value of $\alpha$.

Most of the proof remains the same, and in this section we focus on the required changes.
\begin{lemma}[Descent Lemma]
    \label{lem:descent_no_nSG}
    Assume that $\step < 1/L$, for any $t_0$ and $t$.
    When $d \Delta^2 > \sigma^2 / (\tdelta B)$, with probability at least $1 - \tdelta$, for some constant $c$:
    \[
        f(\x{t_0}) - f(\x{t_0 + t})
        \ge \frac{\step}{16} \sum_{\tau=t_0}^{t_0 + t - 1} \SqrNorm{\nabla f(\x{\tau})} - c \step d \Delta^2 (\step L t + \logErr).
    \]
\end{lemma}

\begin{proof}
    W.l.o.g., we consider $t_0 = 0$.
    Let $\sgdNoise{\tau} = \nabla f_{\batch{t}}(\x{\tau}) - \nabla f(\x{\tau})$ be the stochastic noise at iteration $\tau$.
    By the folklore Descent Lemma,
    \begin{align*}
        f(\x{\tau}) - f(\x{\tau+1})
        &\ge \step \InnerProd{\nabla f(\x{\tau}), \nabla f_{\batch{t}}(\x{\tau})}
            - \frac{L}{2} \SqrNorm{\x{\tau + 1} - \x{\tau}} \\
        &= \step \InnerProd{\nabla f(\x{\tau}), \nabla f(\x{\tau}) + \sgdNoise{\tau} + \ourNoise{\tau}}
            - \frac{L\step^2}{2} \SqrNorm{\nabla f(\x{\tau}) + \sgdNoise{\tau} + \ourNoise{\tau}} \\
        &\ge \step \SqrNorm{\nabla f(\x{\tau})} + \step \InnerProd{\nabla f(\x{\tau}), \sgdNoise{\tau} + \ourNoise{\tau}}
            - \frac{\step^2 L}{2} \pars{\frac{3}{2}\SqrNorm{\nabla f(\x{\tau})} + 3 \SqrNorm{\sgdNoise{\tau} + \ourNoise{\tau}}}
        ,
    \end{align*}
    where we used inequality $\SqrNorm{a + b} \le (1 + \nu) \SqrNorm{a} + (1 + 1/\nu) \SqrNorm{b}$ for $\nu = 1/2$.
    By telescoping, we have:
    \begin{align*}
        f(\x{0}) - f(\x{t})
        &\ge \frac{\step}{4} \sum_{\tau=0}^{t-1} \SqrNorm{\nabla f(\x{\tau})}
            + \step \sum_{\tau=0}^{t-1} \InnerProd{\nabla f(\x{\tau}), \sgdNoise{\tau}}
            + \step \sum_{\tau=0}^{t-1} \InnerProd{\nabla f(\x{\tau}), \ourNoise{\tau}}
            - 3 \step^2 L \sum_{\tau=0}^{t-1} \SqrNorm{\sgdNoise{\tau}}
            - 3 \step^2 L \sum_{\tau=0}^{t-1} \SqrNorm{\ourNoise{\tau}}
    \end{align*}
    The terms containing $\ourNoise{\tau}$ can be bounded as in \Cref{lem:descent}, giving
    \begin{align*}
        f(\x{0}) - f(\x{t})
        &\ge \frac{\step}{8} \sum_{\tau=0}^{t-1} \SqrNorm{\nabla f(\x{\tau})}
            + \step \sum_{\tau=0}^{t-1} \InnerProd{\nabla f(\x{\tau}), \sgdNoise{\tau}}
            - 3 \step^2 L \sum_{\tau=0}^{t-1} \SqrNorm{\sgdNoise{\tau}}
            - c \step d \Delta^2 (\step L t + \logErr)
    \end{align*}
    with probability at least $1 - \tdelta$.
    It remains to bound the terms containing $\sgdNoise{\tau}$.

    \paragraph{Bounding $3 \step^2 L \sum_{\tau=0}^{t-1} \SqrNorm{\sgdNoise{\tau}}$.}
    Since $\ExpSqrNorm{\sgdNoise{\tau}} \le \sigma^2 / B$, by the Markov's inequality, with probability at least $1-\tdelta$ we have
    \[
        3 \step^2 L \sum_{\tau=0}^{t-1} \SqrNorm{\sgdNoise{\tau}}
        \le 3 \step^2 L t \sigma^2 / (\tdelta B)
    \]

    \paragraph{Bounding $\step \sum_{\tau=0}^{t-1} \InnerProd{\nabla f(\x{\tau}), \sgdNoise{\tau}}$.}
    We have $\Var{\InnerProd{\nabla f(\x{\tau}), \sgdNoise{\tau}} \mid \x{\tau}} \le \SqrNorm{\nabla f(\x{\tau})} \cdot \sigma^2 / B$.
    By the Chebyshev's inequality, with probability at least $1 - \tdelta$, we have
    \begin{align*}
        \abs{\step \sum_{\tau=0}^{t-1} \InnerProd{\nabla f(\x{\tau}), \sgdNoise{\tau}}}
        &\le \step \sqrt{\sum_{\tau=0}^{t-1} \SqrNorm{\nabla f(\x{\tau})}} \sigma / \sqrt{\tdelta B} \\
        &\le \frac{\step}{\nu} \sum_{\tau=0}^{t-1} \SqrNorm{\nabla f(\x{\tau})} + \step \nu \sigma^2 / (\tdelta B)
    \end{align*}
    for any $\nu$.
    Taking $\nu = 16$ and using $d \Delta^2 / (\tdelta B)$ finishes the proof.
    


\end{proof}

\begin{lemma}
    \label{lem:error_terms_no_nSG}
    Let $\mathcal{E}_t$ be the event ``$\|\x{0} - \x{\tau}\| < \escRad$ and $\|\x{0} - \tx{\tau}\| < \escRad$ for all $\tau < t$''.
    Then, when $\tdelta \Delta^2 \ge \sigma^2 / B$, for some constant $c$ we have
    \[
        \P{\mathcal{E}_\escIter
            \implies \max\nolimits_{\tau=0}^{\escIter - 1} \norm{\sgdErr{\tau} + \quadErr{\tau}}
            \le 
            c \beta_\tau \step \Delta \sqrt{\logErr}
        } \ge 1 - \escIter \tdelta
    \]
\end{lemma}
    
\begin{proof}
    As in \Cref{lem:error_terms}, the proof is by induction on $t$, and we bound $\dif{t}$ and $\quadErr{t}$ as in \Cref{lem:error_terms}.
    It remains to bound $\sgdErr{t}$.
    %
    By \Cref{ass:grad_concentration}, each $\sgdErr{t}$ has variance $\sigma^2 / B$.
    Hence, by the Chebyshev's inequality, with probability at least $1 - \tdelta$ we have
    \begin{align*}
        \norm{\sgdErr{t}}
        &= \norm{\step \sum_{\tau = 0}^{t - 1} (\id - \step H)^{t - 1 - \tau} \difsgdNoise{\tau}} \\
        &\le \step \sqrt{\sum_{\tau = 0}^{t - 1} (1 + \step \eigen)^{2(t - 1 - \tau)} \frac{\sigma^2}{B \tdelta}} \\
        &\le \step \beta_t \cdot \Delta,
    \end{align*}
    where the last inequality follows from assumption $\Delta^2 \ge \sigma^2 / {B \tdelta}$.
\end{proof}